\documentclass{article}
\usepackage{enumerate}
\usepackage{amsmath}
\usepackage{amsthm}
\usepackage{amssymb}
\usepackage{amsfonts}
\usepackage{color}
\usepackage{enumitem}
\usepackage{forest}
\usepackage{wrapfig}
\usepackage{graphicx}
\usepackage[normalem]{ulem}

\usepackage{dot2texi}
\usepackage{tikz}
\usetikzlibrary{shapes,arrows}
\newtheorem{definition}{Definition}
\newtheorem{theorem}{Theorem}

\newtheorem{lemma}{Lemma}

\newcommand{\sys}[1]{\ensuremath{\mathbf{#1}}}

\newcommand{\ra}{\ensuremath{\rightarrow}}
\newcommand{\Ra}{\ensuremath{\Rightarrow}}

\newcommand{\nc}{\mathop{\mid\!\sim}}

\newcommand{\Dc}{\ensuremath{\mathcal{D}}}

\newcommand{\Fc}{\ensuremath{\mathcal{F}}}

\newcommand{\Ec}{\ensuremath{\mathcal{E}}}

\newcommand{\Arg}{\ensuremath{{\tt Arg}}}
\newcommand{\Con}{\ensuremath{{\tt Conc}}}

\newcommand{\Sub}{\ensuremath{{\tt Sub}}}
\newcommand{\HSub}{\ensuremath{{\tt HSub}}}
\newcommand{\HArg}{\ensuremath{{\tt HArg}}}

\newtheorem{remark}{Remark}
\newtheorem{example}{Example}

\newcounter{Gcount}

\begin{document}
\sloppy
\setlist[enumerate]{itemsep=1pt,leftmargin=*,topsep=1pt}
\setlist[itemize]{itemsep=1pt,leftmargin=*,topsep=1pt}

%\setcopyright{acmcopyright}
%\setcopyright{acmlicensed}
%\setcopyright{rightsretained}
%\setcopyright{usgov}
%\setcopyright{usgovmixed}
%\setcopyright{cagov}
%\setcopyright{cagovmixed}

% DOI
%\doi{http://dx.doi.org/xx.xxxx/xxxxxxx.xxxxxxx}

\title{Reasoning by Cases in Structured Argumentation.}
\author{Mathieu Beirlaen, Jesse Heyninck and Christian Stra{\ss}er\\ \small Ruhr University Bochum \\ \small mathieubeirlaen@gmail.com, jesse.heyninck@gmail.com, \small christian.strasser@rub.de\thanks{Research for this article was sponsored by a Sofja Kova-\break levskaja award of the Alexander von Humboldt Foundation, funded by the German Ministry for Education and Research. We are indebted to Leon van der Torre and Emil Weydert for helpful comments and suggestions.}}
\maketitle

\begin{abstract}
We extend the $ASPIC^+$ framework for structured argumentation so as to allow applications of the reasoning by cases inference scheme for defeasible arguments. Given an argument with conclusion `$A$ or $B$', an argument based on $A$ with conclusion $C$, and an argument based on $B$ with conclusion $C$, we allow the construction of an argument with conclusion $C$. We show how our framework leads to different results than other approaches in non-monotonic logic for dealing with disjunctive information, such as disjunctive default theory or approaches based on the OR-rule (which allows to derive a defeasible rule `If ($A$ or $B$) then $C$', given two defeasible rules `If $A$ then $C$' and `If $B$ then $C$'). We raise new questions regarding the subtleties of reasoning defeasibly with disjunctive information, and show that its formalization is more intricate than one would presume.
\end{abstract}

\section{Introduction}\label{sec_intro}

When formulated in terms of the material implication connective `$\supset$', the pattern of \emph{reasoning by cases} is valid in classical logic:
\begin{equation}
\varphi_1\vee\ldots\vee\varphi_n, \varphi_1\supset\psi,\ldots,\varphi_n\supset\psi \vdash \psi \tag{RBC$_{\supset}$}
\end{equation}
Many formalisms of non-monotonic logic likewise allow \emph{defeasible} applications of reasoning by cases, where a formula `$\varphi\Rightarrow\psi$' reads `If $\varphi$, then normally/usually/probably $\psi$':
\begin{equation}
\varphi_1\vee\ldots\vee\varphi_n, \varphi_1\Ra\psi,\ldots,\varphi_n\Ra\psi \nc \psi \tag{RBC$_{\Ra}$}
\end{equation}
In the context of formal argumentation, it is natural to include a more general argumentative version of the RBC rule. Relative to a knowledge base this rule would allow the construction of an \emph{rbc-
argument} $C$ with conclusion $\psi$ given 
(i) an argument $A$ with the disjunctive conclusion $\varphi_1\vee\ldots\vee\varphi_n$, and 
(ii) for each $i\in\{1,\ldots,n\}$, an argument $B_i$ with conclusion $\psi$ based on an extended knowledge base including $\varphi_i$.  
To the best of our knowledge, no such rule has yet been introduced and studied in the context of formal argumentation. The aim of this paper is to do exactly this, and to investigate the nature of `argumentation by cases'.

For defining rbc-arguments and attacks on rbc-arguments, we will extend the $ASPIC^+$ framework \cite{ModPra13}. $ASPIC^+$ is a framework for instantiating abstract argumentation frameworks as conceptualized by Dung \cite{Dun95}. We introduce abstract argumentation frameworks in Section \ref{sec_abstractarg}, and define our formalism in Section \ref{sec_aspichyp}. Our approach is limited in at least two ways. First, we do not allow for nested or iterated rbc-arguments. Second, we do not yet take into account priorities assigned to arguments, nor do we include undercutting attacks. The removal of these limitations is left for future work.

The present approach raises new questions regarding the nature of arguing by cases. For instance, what happens if in the rbc-argument $C$ one of the arguments $B_i$ is rebutted by an independent argument? Is this a sufficient condition for rejecting argument $C$? After all, there are $n-1$ other defeasible `paths' leading to $C$'s conclusion. Our formalism implements a cautious rationale according to which a successful rebut on one of its paths is indeed sufficient for rejecting an rbc-argument. 
In Section \ref{sec_related} we show how this approach leads to intuitive outcomes different from those obtained by other formalisms in non-monotonic logic. 

%%% Local Variables:
%%% mode: latex
%%% TeX-master: "jelia-rbc"
%%% End:

\section{Abstract argumentation}\label{sec_abstractarg}

A Dung-style \emph{abstract argumentation framework} (AF) is a pair $(\Arg,\tt{Att})$ where $\Arg$ is a set of arguments and $\tt{Att} \subseteq \Arg\times\Arg$ is a binary relation of attack. Relative to an AF, Dung defines a number of extensions -- subsets of $\Arg$ -- on the basis of which we can evaluate the arguments in $\Arg$. 
\begin{definition}[Defense]
A set of arguments $\mathcal{X}$ defends an argument $A$ iff every attacker of $A$ is attacked by some $B \in\mathcal{X}$.
\end{definition}
\begin{definition}[Extensions]\label{def_af}
  Let $(\Arg, {\tt Att})$ be an AF.
  %For any $a \in \mathcal{A}$, $a$ is acceptable w.r.t.\ some $\mathcal{S} \subseteq \mathcal{A}$ iff for all $b$ such that $(b,a) \in {\sf Att}$ there is a $c \in \mathcal{S}$ for which $(c,b) \in {\sf Att}$.\\
If $\mathcal{E} \subseteq \Arg$ is conflict-free, i.e.\ there are no $A,B\in \mathcal{E}$ for which $(A,B)\in {\tt Att}$, then (i) $\mathcal{E}$ is a \emph{complete extension} iff $A \in \mathcal{E}$ whenever $\mathcal{E}$ defends $A$; (ii) $\mathcal{E}$ is a \emph{preferred extension} iff it is a set inclusion maximal complete extension;
and (iii) $\mathcal{E}$ is the \emph{grounded extension} iff it is the set inclusion minimal complete extension.
%\item $S$ is a stable extension iff it is preferred and $\forall{Y}\not\in S, \exists{X}\in S$ s.t.\ $(X,Y)\in\mathcal{D}$.
\end{definition}
\noindent Dung \cite{Dun95} showed that for every AF there is a \emph{unique} grounded extension. % , and it can be constructed as follows.
% \begin{definition}[Construction of grounded extension]\label{def_ground2}
% The grounded extension $\mathcal{G}$ relative to an AF $(\mathcal{A}, {\sf Att})$ is defined as follows:
% \begin{itemize}
% \item $\mathcal{G}_0$: the set of all arguments in $\mathcal{A}$ without attackers;
% \item $\mathcal{G}_{i+1}$: all arguments defended by $\mathcal{G}_i$;
% \item $\mathcal{G} = \bigcup_{i\geq 0} \mathcal{G}_i$.
% \end{itemize}
% \end{definition}
On Dung's abstract approach from \cite{Dun95}, arguments are basic units of analysis the internal structure of which is not represented. In what follows we will \emph{instantiate} abstract arguments by allowing for the representation of their internal logical structure. 
%In doing so, we extend the $ASPIC^+$ framework  from \cite{ModPra13} with hypothetical arguments.

%In $ASPIC^+$, arguments are inference trees obtained via the application of strict and/or defeasible rules.
%%% Local Variables:
%%% mode: latex
%%% TeX-master: "jelia-rbc"
%%% End:

\section{Argumentation by cases}\label{sec_aspichyp}

In this section we define structured argumentation frameworks (SAFs) for reasoning by cases. Our point of departure is an instantiation of the $ASPIC^+$ framework (without priorities, without defeasible premises and without undercuts).

%\jesse{\sout{We adjust $ASPIC^+$ in the following ways. (i) We define a new type of argument called an \emph{rbc-argument} (see point (iv) in Definition \ref{def_arg}). By means of rbc-arguments, we can argue by cases in an argumentation formalism.
%(ii) We generalize the attack relation so as to include argumentation by cases, using the concept of a \emph{hypothetical argument} (Definition \ref{def_hyp}), and we define a logical consequence relation for the resulting SAFs.} 

%\sout{In the remainder of this section we define and illustrate the construction of arguments (Section \ref{sub_argdefs}) and argumentative attacks (Section \ref{sub_attacks}); we define consequence relations for the resulting argumentation frameworks (Section \ref{sub_conseq}); and we briefly discuss the meta-theoretical properties of our framework (Section \ref{sub_ratpostulates}).}}
We adjust $ASPIC^+$ in the following ways. (i) We define a new type of argument called an \emph{rbc-argument} (see point 4 in Definition \ref{def_arg}). By means of rbc-arguments, we can argue by cases in an argumentation formalism (Section \ref{sub_argdefs}).
(ii) We generalize the attack relation so as to include argumentation by cases, using the concept of a \emph{hypothetical argument} (Section \ref{sub_attacks}), and we define a logical consequence relation for the resulting SAFs (Section \ref{sub_conseq}). We briefly discuss the meta-theoretical properties of our framework (Section \ref{sub_ratpostulates}).
%\subsection{Definitions}\label{sub_definitions}

\subsection{Arguments}\label{sub_argdefs}
\def\AT{{\rm AT}}

We illustrate our framework using the propositional fragment of classical logic (\sys{CL}) as our core logic. We denote the consequence relation of \sys{CL} by $\vdash$.
To obtain the formal language $\mathcal{L}$ of \sys{CL}, we close a denumerable stock $\mathcal{P}=\{p,q,r,\ldots\}$ of propositional letters under the usual \sys{CL}-connectives $\neg, \vee, \wedge, \supset, \equiv$. We also add the verum constant $\top$ and the falsum constant $\bot$ to $\mathcal{L}$. For reasons of transparency we will sometimes use subscripted letters $p_1,p_2,q_1,q_2,\ldots$ as names for propositional letters.

\begin{definition}[Argumentation theory]\label{def_AT} An argumentation theory (AT) is a triple 
$\AT = (\mathcal{L},\mathcal{R},\mathcal{F})$ where:
\begin{itemize}[itemsep=1pt]
\item $\mathcal{L}$ is our formal language defined above;
\item $\mathcal{R}=\mathcal{S}\cup\mathcal{D}$ is a set of strict ($\mathcal{S}$) and defeasible ($\mathcal{D}$) inference rules of the form $\varphi_1,\ldots,\varphi_n \ra \psi$ and $\varphi_1,\ldots,\varphi_n \Ra \psi$ respectively (where $\varphi_1,\ldots,\varphi_n,\psi\in\mathcal{L}$); and
\item $\Fc\subseteq\mathcal{L}$ is a \sys{CL}-consistent knowledge base.\footnote{$\Fc\subseteq\mathcal{L}$ is \emph{\sys{CL}-consistent} iff $\Fc\not\vdash\bot$.}
\end{itemize}
\end{definition}
We assume in addition that $\varphi_1,\ldots,\varphi_n\ra \psi \in \mathcal{S}$ iff $\{\varphi_1,\ldots,\varphi_n\}\vdash \psi$. % \footnote{It suffices to take into account finite premise sets since \sys{CL} is compact: If $\Fc\vdash \varphi$, then $\Fc'\vdash \varphi$ for some finite $\Fc'\subseteq\Fc$.}
Since we keep $\mathcal{L}$ and $\mathcal{S}$ fixed, we will in the remainder refer to ATs as pairs $(\mathcal{D,F})$.

\begin{definition}[Arguments]\label{def_arg}
  Given an argumentation theory $\AT = (\Dc,\Fc)$, the set of arguments $\Arg^{\bot}(\AT)$ contains:
\begin{enumerate}[itemsep=1pt]
\item \(A = \langle \phi \rangle\) where \(\phi \in \Fc\)
\begin{itemize}[itemsep=1pt,topsep=1pt]
\item \(\Con(A) = \phi\)
\item \(\Sub(A) = \lbrace A \rbrace\)
\item \(\HSub(A) = \emptyset\)
\end{itemize}
\item \(A = \langle A_{1} , \ldots, A_n \rightarrow \phi \rangle\) where \(A_1, \ldots, A_n \in \Arg^\bot(\AT)\) and
\(\Con(A_1), \ldots, \Con(A_n) \ra \phi\in \mathcal{S}\)
\begin{itemize}[itemsep=1pt,topsep=1pt]
\item \(\Con(A) = \phi\)
\item \(\Sub(A) = \lbrace A \rbrace \cup \Sub(A_1) \cup \ldots \cup \Sub(A_n)\)
\item \(\HSub(A) = \HSub(A_1) \cup \ldots \cup \HSub(A_n)\)
\end{itemize}
\item \(A = \langle A_{1} , \ldots, A_n \Rightarrow \phi \rangle\) where \(A_1, \ldots, A_n \in \Arg^\bot(\AT)\) and
\(\Con(A_1), \ldots, \Con(A_n) \Rightarrow \phi \in \Dc\)
\begin{itemize}[itemsep=1pt,topsep=1pt]
\item \(\Con(A) = \phi\)
\item \(\Sub(A) = \lbrace A \rbrace \cup \Sub(A_1) \cup \ldots \cup \Sub(A_n)\)
\item \(\HSub(A) = \HSub(A_1) \cup \ldots \cup \HSub(A_n)\)
\end{itemize}
\item \(A = \langle A_{1} , [A_2], \ldots, [A_n] \leadsto \phi\rangle\) where $n \ge 3$
\begin{itemize}[itemsep=1pt,topsep=1pt]
\item \(\phi = \bigvee \lbrace \Con(A_2), \ldots, \Con(A_n) \rbrace\), 
\item \(\Con(A_1) = \bigvee_{i=2}^n \psi_i\), \(A_i \in \Arg^\bot((\Dc, \Fc \cup \lbrace \psi_i \rbrace))\), 
\item  and $\HSub(A_i)=\emptyset$ for all \(2 \le i \le n\).
\end{itemize}
We have:
\begin{itemize}[itemsep=1pt,topsep=1pt]
\item \(\Con(A) = \phi\)
\item \(\Sub(A) = \lbrace A \rbrace \cup \Sub(A_1)\)
\item \(\HSub(A) = \HSub(A_1) \cup \lbrace (A_2, \psi_2), \ldots, (A_n, \psi_n) \rbrace\)
\end{itemize}
\end{enumerate}
\end{definition}
% To avoid clutter we will often drop $\langle$ and $\rangle$ when there is no risk of confusion.

Definition \ref{def_arg} departs in two respects from the way arguments are usually defined in $ASPIC^+$. First, there is a new class of arguments constructed by means of rule 4: these arguments are called \emph{rbc-arguments}. They correspond to applications of the reasoning by cases scheme outlined in Section \ref{sec_intro}.
\begin{example}
Let \(\Dc = \lbrace \top \Rightarrow p \vee q; \:\: p \Rightarrow p_1; \:\: p_1 \Rightarrow p_2; \:\: p_2 \Rightarrow r; \:\: q \Rightarrow q_1; \:\: q_1 \Rightarrow r \rbrace\) and $\AT = (\Dc, \{\top\})$. The following are arguments in \(\Arg^{\bot}(\AT)\):
\begin{itemize}[itemsep=1pt,label=,leftmargin=*]
\item \(A_1 =\langle \langle \top\rangle \Rightarrow p \vee q \rangle\)
\item \(A_2 = \langle A_1, [\langle\langle \langle p\rangle \Rightarrow p_1 \rangle \Rightarrow p_2\rangle \Rightarrow r], [\langle\langle q\rangle \Rightarrow q_1\rangle\Rightarrow r] \leadsto r \rangle\)
\item \(A_3 = \langle A_{1}, [\langle p \rangle \Rightarrow p_1], [\langle q \rangle \Rightarrow q_1] \leadsto p_1 \vee q_1 \rangle\)
\end{itemize}
%$\HSub(A_1)=\{(\langle\langle \langle p\rangle \Rightarrow p_1 \rangle \Rightarrow p_2\rangle \Rightarrow r],p),( \langle q\rangle \Rightarrow r, q)\}$.
\end{example}
$A_2$ and $A_3$ are rbc-arguments constructed on the basis of the `cases' $p$ and $q$ in the disjunctive conclusion of $A_1$. Argument $A_2$ concludes that $r$, while $A_3$ concludes that $p_1\vee q_1$.

The second novel feature of Definition \ref{def_arg} is that we not only keep track of an argument $A$'s conclusion ($\Con(A)$) and its sub-arguments ($\Sub(A)$), but also of its \emph{hypothetical sub-arguments} ($\HSub(A)$). These are pairs consisting of an argument $B$ and a formula $\phi$, where $B$ is constructible on the basis of the extended AT obtained by adding $\phi$ to the knowledge base of the original AT.
For instance, $\HSub(A_2)= \bigl\{ \bigl(\langle\langle\langle\langle p\rangle \Rightarrow p_1 \rangle \Rightarrow p_2\rangle \Rightarrow r\rangle, p \bigr),
  \bigl(\langle\langle\langle q\rangle \Rightarrow q_1\rangle\Rightarrow r\rangle, q \bigr) \bigr\}$.
  
%Hypothetical sub-arguments need not themselves be constructible on the basis of the original AT. Instead, they are constructible on the basis of the AT obtained by adding their factual basis to the knowledge base of the original AT. For instance,

$A_2$'s hypothetical sub-argument $(\langle\langle\langle\langle p\rangle \Rightarrow p_1 \rangle \Rightarrow p_2\rangle \Rightarrow r\rangle,p)$ contains the argument $\langle\langle\langle\langle p\rangle \Rightarrow p_1 \rangle \Rightarrow p_2\rangle \Rightarrow r\rangle$ constructible on the basis of the AT $(\mathcal{D},\{\top, p\})$.
%\jesse{Requiring that the hypothetical sub-arguments used in the construction of an rbc-argument themselves do not contain any hypothetical sub-arguments themselves precludes nested rbc-arguments from occurring. For simplicity and due to spatial limitations, the investigation of nested rbc-arguments is left for future work.}
\begin{remark}
If one is interested in reducing the size of \(\Arg^{\bot}(\AT)\), one may only allow for minimal disjunctions when generating arguments of type 4. More precisely, where \(A_{1}\) is of the form \(\langle B_1, \ldots, B_m \rightarrow \bigvee_{i=1}^{n} \psi_i \rangle\) or \(\langle B_1, \ldots, B_m \Rightarrow \bigvee_{i=1}^{n} \psi_i \rangle\), \(A = \langle A_{1} , [A_2], \ldots, [A_n] \leadsto \phi\rangle \in \Arg^{\bot}(\AT)\) only if there is no \(\bigvee_{j \in J} \psi_j\) where \(J \subset \lbrace 2, \ldots, n \rbrace\) for which \(\Con(B_{1}), \ldots, \Con(B_{m}) \rightarrow \bigvee_{j \in J} \psi_{j} \in \mathcal{S}\).
\end{remark}
%\subsubsection{Contamination}
Definition \ref{def_arg} allows for the construction of arguments containing sub-arguments the conclusions of which are jointly inconsistent, such as arguments $A_1$ and $A_2$ in the following example, both of which rely on both $p$ and $\neg p$ in their construction. 
\begin{example}
  Let $\AT = (\{\top \Ra s,\top\Ra p, p\Ra \lnot p\}, \{\top\})$.
  \[ \begin{array}{l@{\ \ \ \ \ \ \ \ }l}
       A_0 = \langle \langle\top\rangle\Ra p\rangle  & A_2 =  \langle A_0,A_1\ra \lnot s\rangle \\[1ex]
       A_1 = \langle A_0\Ra \lnot p\rangle & A_3 = \langle\langle\top\rangle \Ra s\rangle \\
     \end{array}       
   \]
  
% \begin{itemize}[itemsep=1pt,label=]
% \item $A_0 =$ $\top\Ra p$
% \item $A_1 =$ $A_0\Ra \lnot p$
% \item $A_2 =$ $A_0,A_1\ra \lnot s$
% \item $A_3 =$ $\top \Ra s$
% \end{itemize}   
%\begin{tikzpicture}
%\tikzset{vertex/.style = {shape=rectangle,draw,minimum size=1.5em,rounded corners=.8ex}}
%\tikzset{edge/.style = {->,> = latex'}}
%
%\node[vertex] (0)  at (2,2) [draw, align=left] {$A_0$};
%\node[vertex] (1)  at (0,2) [draw, align=left] {$ A_1$};
%\node[vertex] (2)  at (1,1) [draw, align=left] {$ A_2$};
%\node[vertex] (3)  at (1,0) [draw, align=left] {$ A_3$};
%
%
%\draw[edge, dotted, bend left] (0) to (1);
%\draw[edge, dotted, bend left] (1) to (0);
%\draw[edge, dotted] (0) to (2);
%\draw[edge, dotted] (1) to (2);
%\draw[edge, dotted] (2) to (3);
%
%\end{tikzpicture}
\end{example}
When \sys{CL} is used as the underlying logic, inconsistent arguments like $A_1$ and $A_2$ may contaminate our formalism by blocking intuitively acceptable arguments like $A_3$. (This is because the conclusions of $A_2$ and $A_3$ are conflicting, causing $A_2$ to attack and exclude $A_3$, cfr.\ infra.) Contamination problems of this kind have been studied and tackled in $ASPIC^+$ \cite{CamCarDun12,Wu12}. We can avoid them by filtering out inconsistent arguments.

\begin{definition}\label{def_dagger}
We define \(\dagger A\) for an argument \(A\) recursively as follows:

\begin{itemize}[itemsep=1pt]
\item \(\dagger\langle\phi\rangle = \phi\)
\item \(\dagger A = \phi \wedge \dagger B_1 \wedge \ldots \wedge \dagger B_n\) where \(A = \langle B_1, \ldots, B_n \rightarrow \phi \rangle\)
\item \(\dagger A = \phi \wedge \dagger B_1 \wedge \ldots \wedge \dagger B_n\) where \(A = \langle B_1, \ldots, B_n \Rightarrow \phi \rangle\)
\item \(\dagger A = \phi \wedge \dagger B_1 \wedge (\dagger B_2 \vee \ldots \vee \dagger B_n)\) where \(A = \langle B_1, [B_2],\ldots [B_n] \leadsto \phi \rangle\).
\end{itemize}
\end{definition}
\begin{definition}\label{def_inconsistent}
An argument \(A\) is \emph{inconsistent} iff \(\dagger A \vdash \bot\). Otherwise \(A\) is \emph{consistent}. Relative to  $\AT = (\Dc,\Fc)$, \(\Arg(\AT)\) is \(\Arg^{\bot}(\AT)\) without inconsistent arguments. 
\end{definition}

For arguments without occurrences of \(\leadsto\) our definition of inconsistent arguments is equivalent to that of \cite{Wu12}. In the remainder we will focus on the set  \(\Arg(\AT)\) rather than \(\Arg^{\bot}(\AT)\), avoiding contamination problems.

\subsection{Attacks}\label{sub_attacks}
\def\SAF{{\rm SAF}}
\def\AT{{\rm AT}}

%Argumentative attacks express that two arguments contain contradictory information.

In $ASPIC^+$, attacks are defined in terms of a generic contrariness operator. We define them in terms of `$\neg$', so that arguments the conclusions of which are classical contradictories attack each other. We have to be careful when defining argumentative attacks when rbc-arguments are involved, since we must take into account the hypothetical sub-arguments of an rbc-argument. New questions arise here. For instance, an argument's hypothetical sub-argument may conflict with a non-hypothetical argument.

\begin{example}\label{motivation_att_1}
Let $\AT = (\mathcal{D,F})$, with  $\mathcal{D}=\{p\Ra q\vee r; \ q\Ra s; \ s\Ra v; \ r\Ra v; \ t\Ra \lnot s\}$ and $\mathcal{F}=\{p,t\}$.
\begin{itemize}[itemsep=1pt,label=]
\item $A_1= \bigl\langle\langle\langle p\rangle \Ra q\lor r\rangle,[\langle\langle q\rangle\Ra s\rangle\Ra v],[\langle r\rangle\Ra v]\leadsto v \bigr\rangle$
\item $A_2= \bigl\langle\langle t\rangle\Ra \lnot s \bigr\rangle$
\end{itemize}
\end{example}
In Example \ref{motivation_att_1}, the non-hypothetical argument $A_2$ is in conflict with the argument $\langle\langle q\rangle\Ra s\rangle\Ra v$, which belongs to $A_1$'s hypothetical sub-argument $(\langle\langle\langle q\rangle\Ra s\rangle\Ra v\rangle,q)$. The desirable outcome in this example is that $A_2$ attacks $A_1$, but not vice versa.
%From the strict premise $t$ we can defeasible derive $\lnot s$, which is in conflict with the hypothetical argument $q\Ra s\Ra v$. 

As a further illustration, consider the following scenario.
%There is an additional situation we have to take into account. Sometimes it might be the case that, assuming a disjunct $\phi$ used in an rbc-argument $A$, we can derive an additional argument that is in conflict with the hypothetical argument derived on the basis of $\phi$ in the rbc-argument $A$:

\begin{example}\label{motivation_att_2}
 $\AT = (\mathcal{D},\mathcal{F})$, with  $\mathcal{D} = \{p\Ra (q\lor r); \ q\Ra s_1; \  s_1\Ra s_2; \ s_2\Ra v; \ r\Ra v; \ q\Ra \lnot s_1\}$ and $\mathcal{F} = \{p\}$. 
\begin{itemize}[itemsep=0pt,label=]
\item $A_0 = \langle\langle p \rangle\Rightarrow q\vee r\rangle$ 
\item $A_1 = \langle A_0, [\langle\langle\langle q\rangle\Rightarrow s_1\rangle\Rightarrow s_2\rangle\Rightarrow v], [\langle r\rangle\Rightarrow v] \leadsto v \rangle$ 
\end{itemize}
The following argument is constructible on the basis of the extended theory $\AT' = (\mathcal{D},\mathcal{F}\cup\{q\})$:
\begin{itemize}[itemsep=0pt,label=]
\item $A_2 =$ $\langle\langle q\rangle \Rightarrow \neg s_1\rangle$.
\end{itemize}  
\end{example}
$A_1$ contains the intermediate conclusion $s_2$ based on the assumption $q$. However, $A_2$ concludes that $\neg s_2$ on the basis of the same assumption, $q$. The desirable outcome in this example is to let $A_1$ and $A_2$ attack each other, since these arguments were both constructed on the basis of our knowledge base plus the assumption that $q$, and since their conclusions are contradictories.

To handle examples like these, we introduce the set $\HArg(\Dc,\Fc)$ of all arguments that can be constructed on the basis of some disjunct used in the construction of an rbc-argument in $\Arg(\Dc,\Fc)$.

% \begin{definition} Where $\mathcal{A}$ is a set of arguments and $\phi$ a formula, 
%   $\pi_1((\mathcal{A},\phi))= \mathcal{A}$ and $\pi_2((\mathcal{A},\phi)) = \phi$. % We lift the definition as usual: where $i \in \{1,2\}$, 
%  % $\pi_i(\Delta)=\{\pi_i((\mathcal{A},\phi))\mid (\mathcal{A},\phi) \in \Delta\}$.
% \end{definition}

Where $A$ is an argument and $\phi$ a formula, 
 $\pi_1(A,\phi)= A$ and
$\pi_2(A,\phi) = \phi$. We lift the definition as usual: where $i \in \{1,2\}$, 
%  %
$\pi_i(\Delta) = \{\pi_i(A,\phi) \mid (A,\phi) \in \Delta\}$.

Where $\phi \in \mathcal{L} \setminus \Fc$, $\AT = (\Dc,\Fc)$ and $\AT' = (\Dc,\Fc \cup \{\phi\})$, we denote $\Arg(\AT') \setminus \Arg(\AT)$ by $\Arg^{\phi}(\AT)$.

\begin{definition}[Hypothetical Arguments]\label{def_hyp} Where $\AT = (\Dc,\Fc)$, 
  \(\HArg(\AT)\) is the set of all \(A \in \Arg^{\phi}(\AT)\) such that \(\phi \in \pi_2(\HSub(B))\) for some $B\in \Arg(\AT)$.
\end{definition}

% Note that for any $\Dc$ and any $\Fc$, $\Arg(\Dc,\Fc)\subseteq \HArg(\Dc,\Fc)$.

\begin{definition}[Attacks,  Rebuts]\label{def_attack} For a given theory $\AT$, we define a direct attack relation
  \begin{multline*}
  {\tt Att}(\AT) \subseteq (\Arg(\AT) \times \Arg(\AT))  \\ \cup (\Arg(\AT) \times  \HArg(\AT))  \\ \cup ( \HArg(\AT) \times \HArg(\AT))
\end{multline*}
  as follows:
\(A\) directly attacks \(B\) iff \(B\) is of the form \(\langle \ldots \Rightarrow \Con(B) \rangle\), (\(\Con(A) = \neg\Con(B)\) or \(\Con(B) = \neg\Con(A)\)),  and
\begin{itemize}[itemsep=0pt]
\item \(A \in \Arg(\AT)\) and \(B \in \Arg(\AT)\) or
\item \(A \in \Arg(\AT)\) and \(B \in\HArg(\AT)\) or
\item \(A \in \Arg^{\phi}(\AT)\) and \(B \in \Arg^{\phi}(\AT)\) for some $\phi \in \mathcal{L} \setminus \mathcal{F}$.
\end{itemize}

We lift the definition recursively in the following way: \(A\) attacks \(B\) if \(A\) directly attacks \(B\) or it attacks some \(C \in (\Sub(B) \setminus \lbrace B \rbrace) \cup\bigcup \frak{\pi}_1(\HSub(B))\).
\end{definition}

For an argument $A$ to directly attack an argument $B$, the following requirements need to be fulfilled: The conclusion of $A$ conflicts with the conclusion of $B$ and (either $A$ is non-hypothetical, or $A$ and $B$ are hypothetical arguments based on the same assumption $\phi$).

To illustrate how this works, reconsider our examples. In Example \ref{motivation_att_1}, $\langle\langle q\rangle\Ra s \rangle \in {\sf HArg}(\AT)$ and $A_2$ directly attacks $\langle\langle q\rangle\Ra s\rangle$, so $A_2$ attacks $A_1$ (but not vice versa). In Example \ref{motivation_att_2} the arguments $\langle\langle q\rangle\Rightarrow s_1\rangle$ and $A_2$ are in $\Arg^q(\AT)$, so these arguments directly attack each other. Consequently, $A_2$ also attacks $A_1$.

\subsection{Consequence relations}\label{sub_conseq}
\def\SAF{{\rm SAF}}
\def\AT{{\rm AT}}

%With Definitions \ref{def_arg} and \ref{def_attack} in place, we are ready to construct structured argumentation frameworks (SAFs) relative to an AT. 
%SAFs are triples consisting of two sets of arguments and an attack relation:

\begin{definition} The \it{structured argumentation framework} (in short, SAF) defined by the theory $\AT$ is the pair $( \Arg(\AT)\cup\HArg(\AT), {\tt Att}(\AT))$.
\end{definition}
Given a SAF, we can use the argumentation semantics from Section \ref{sec_abstractarg} to define consequence relations:
\begin{definition} Let $\SAF = (\Arg(\AT)\cup\HArg(\AT)$, ${\tt Att}(\AT))$, let ${\sf sem}\in\{\sf Cmp, \sf Prf, \sf Grd\}$, and let ${\sf Cmp}(\SAF)$, ${\sf Prf}(\SAF)$, and ${\sf Grd}(\SAF)$ denote the sets of $\SAF$'s complete extensions, $\SAF$'s preferred extensions, and $\SAF$'s grounded extension respectively.
\begin{itemize}
% \item $\SAF \vdash^{\cup}_{\sf sem} \phi$ iff there is an $A\in \mathcal{B}\cap{\tt Arg}(\AT)$ with ${\rm conc}(A)=\phi$ for some $\mathcal{B} \in {\sf sem}(\SAF)$.
\item $\SAF \nc^{\cap}_{\sf sem} \phi$ iff for every $\mathcal{B} \in {\sf sem}(\SAF)$ there is an $A\in \mathcal{B}\cap{\tt Arg}(\AT)$ with ${\rm conc}(A)=\phi$.
\item $\SAF \nc^{\Cap}_{\sf sem} \phi$ iff there is a $\mathcal{B} \in \bigcap{\sf sem}(\SAF) \cap \Arg(\AT)$ with ${\rm conc}(A)=\phi$.
\end{itemize}
Since the grounded extension is unique both definitions coincide for ${\sf sem}={\sf Grd}$.
% As the grounded extension $\mathcal{G}$ is unique, it suffices to define: $\SAF \vdash^{\cap}_{\sf grd} \phi$ iff there is an $A\in \Gc\cap{\tt Arg}(\AT)$ with ${\rm conc}(A)=\phi$.
\end{definition}
Relative to a theory $\AT$, the `virtual' arguments in $\HArg(\AT)$ are capable of attacking arguments in $\Arg(\AT)$ (in their hypothetical subarguments), and consequently of preventing the derivability of conclusions of arguments in $\Arg(\AT)$. However, the conclusions of virtual arguments are never themselves derivable from AT.
%%% Local Variables:
%%% mode: latex
%%% TeX-master: "jelia-rbc"
%%% End:

\subsection{Rationality postulates}\label{sub_ratpostulates}

In \cite{CamAmg07,CamCarDun12} several postulates were proposed to evaluate formalisms for structured argumentation.
In the present context these postulates read as follows. Given a SAF $(\Arg(\AT)\cup\HArg(\AT), {\tt Att}(\AT))$ where \(\mathcal{E} \in {\sf Cmp}(\SAF)\):\footnote{The proofs of these properties are to be found in the technical appendix as indicated below, except for te proof of non-interference which we omit due to space restrictions. The authors in \cite{CamAmg07} distinguish between direct and indirect consistency: in our framework these definitions are equivalent since our strict rules are closed under $\mathsf{CL}$.}
\begin{enumerate}[itemsep=1pt,leftmargin=*,topsep=1pt,label=]
\item \textbf{Sub-argument closure}: where \(A \in \mathcal{E}\), \({\tt Sub}(A) \subseteq \mathcal{E}\) (immediate in view of Theorem \ref{thm:closure:sub:prime}) 
\item \textbf{Closure under strict rules}: where \(A_1, \ldots, A_n \in \mathcal{E} \cap \Arg(\AT)\) and \({\rm Conc}(A_1), \ldots, {\rm Conc}(A_n) \vdash B\) also \(\langle A_{1}, \ldots, A_n \rightarrow B \rangle \in \mathcal{E} \cap \Arg(\AT)\) (see Theorem \ref{thm:closure}),
\item \textbf{Consistency}: \(\lbrace {\rm Conc}(A) \mid A \in \mathcal{E} \cap \Arg(\AT) \rbrace\) is consistent (see Theorem~\ref{thm:indirect:con}),
% \item \textbf{Indirect consistency}: the set obtained by closing \(\lbrace {\rm Conc}(A) \mid A \in \mathcal{E} \rbrace\) under the strict rules in \(\mathcal{S}\) is consistent (see Theorem \ref{thm:indirect:con}).
\item \textbf{Non-interference}: Let \({\sf Atoms}(\Fc)\) [\({\sf Atoms}(\Dc)\)] be the set of all atoms occurring in \(\Fc\) [\(\Dc\)]. Where $\AT = (\Dc,\Fc)$, $\AT' = (\Dc \cup \Dc', \Fc \cup \Fc')$, $\SAF = (\Arg(\AT)\cup\HArg(\AT), {\tt Att}(\AT))$, $\SAF' = (\Arg(\AT')\cup\HArg(\AT), {\tt Att}(\AT'))$, ${\nc} \in \lbrace {\nc_{\sf Grd}^{\cap}}$, ${\nc_{\sf Grd}^{\Cap}}$, ${\nc_{\sf Prf}^{\cap}}$, ${\nc_{\sf Prf}^{\Cap}}\rbrace$, ${\sf Atoms}(\phi) \subseteq {\sf Atoms}(\Dc) \cup {\sf Atoms}(\Fc)$, and \(({\sf Atoms}(\Dc) \cup {\sf Atoms}(\Fc)) \cap ({\sf Atoms}(\Dc') \cup {\sf Atoms}(\Fc')) = \emptyset\), we have:
  \begin{equation*}
\SAF \nc \phi \mbox{ iff } \SAF' \nc \phi.
\end{equation*}
The property of non-interference can be used to show that the present framework avoids contamination problems of the kind discussed in Section \ref{sub_argdefs} (see  \cite{CamCarDun12,Wu12}).
% \begin{multline*}
%   Cn_{\nc}(\Dc \cup \Dc', \Fc \cup \Fc')_{\mid {\sf Atoms}(\Dc) \cup {\sf Atoms}(\Fc)} = \\ Cn_{\nc}(\Dc,\Fc)_{\mid {\sf Atoms}(\Dc) \cup {\sf Atoms}(\Fc)}
% \end{multline*}
\end{enumerate}

%%% Local Variables:
%%% mode: latex
%%% TeX-master: "jelia-rbc"
%%% End:

\section{Related work}\label{sec_related}

\subsection{Disjunctive Defaults}
\label{sec:orgheadline1}
In \cite{Gelfond_et-al_PKRR_1991} a generalization of default logic, \emph{disjunctive default logic}, is proposed that is more apt to deal with disjunctions than Reiter's original approach. For instance, given the default theory with \(p \vee q\) and the defaults \(p \Rightarrow r\) and \(q \Rightarrow r\), \(r\) is not a default consequence in Reiter's approach since the only extension of this theory is \(Cn(p\vee q)\). In disjunctive default logic, an alternative disjunction \(\mid\) is available: \(p \mid q\) enforces that \(p\) or \(q\) is in any extension of the theory. So, for the default theory consisting of \(p \mid q\), and the defaults \(p \Rightarrow r\) and \(q \Rightarrow r\) we have two extensions, \(Cn(\{p,r\})\) and \(Cn(\{q,r\})\). Now \(r\) is a skeptical consequence. 
Default consequents can also make use of \(\mid\): a disjunctive default is of the form:
\[ \frac{\phi : \psi_1, \ldots, \psi_n}{\gamma_1 \mid \ldots \mid \gamma_m}\]
where \(\phi\) is the prerequisite, \(\psi_1, \ldots, \psi_n\) are justifications, and \(\gamma_1, \ldots, \gamma_m\) are consequents of the default. A set of formulas \(\Xi\) is an extension of a disjunctive default theory consisting of the disjunctive defaults in \(\Delta\) (we here follow the convention in \cite{Gelfond_et-al_PKRR_1991} according to which 'facts' are considered as disjunctive defaults with empty prerequisite and empty justification) if it satisfies the following requirements: (i) for any \(\frac{\phi : \psi_1, \ldots, \psi_n}{\gamma_1 \mid \ldots \mid \gamma_m} \in \Delta\), if \(\phi \in \Xi\) and \(\neg \psi_1, \ldots, \neg \psi_n \notin \Xi\) then \(\gamma_i \in \Xi\) for some \(1 \le i \le m\), (ii) \(Cn(\Xi) = \Xi\), and (iii) \(\Xi\) is minimal with properties (i) and (ii).\footnote{This definition in \cite{Gelfond_et-al_PKRR_1991} is suboptimal in that it gives undesired results: e.g.\ for the theory $\Delta = \{\frac{\top ~:~ \neg p}{\neg p}\}$ also $Cn(p)$ will form an extension. The problem can easily be fixed though by defining extensions analogous to Reiter.}

We compare our approach to disjunctive default logic by thinking of \(\Rightarrow\) as a default conditional: \(\psi \Rightarrow \phi\) encodes the normal default \(\frac{\psi : \phi}{\phi}\). We start our comparison with the example given above. Let \(\Delta_1 = \lbrace p \mid q, p \Rightarrow r, q \Rightarrow r \rbrace\). \(\Delta_1\) has two extensions, \(Cn(\lbrace p, r \rbrace)\) and \(Cn(\lbrace q, r \rbrace)\) and so \(r\) is a skeptical consequence. This outcome corresponds to our approach for the theory $\AT_1 = (\{p\Rightarrow r, q\Rightarrow r\},\{p\vee q\})$: the argument \(\langle\langle p \vee q\rangle, [\langle p\rangle \Rightarrow r], [\langle q\rangle \Rightarrow r] \leadsto r\rangle\) is in all complete extensions of \(\AT_1\).

%One way to think about \(\Rightarrow\) in our approach is by taking it to be a default conditional: \(\psi \Rightarrow \phi\) encodes the normal default \(\frac{\psi : \phi}{\phi}\). This gives us a motivation to compare our approach with disjunctive defaults. We start with the example given above. Let \(\Delta_1 = \lbrace p \mid q, p \Rightarrow r, q \Rightarrow r \rbrace\). \(\Delta_1\) has two extensions, \(Cn(\lbrace p, r \rbrace)\) and \(Cn(\lbrace q, r \rbrace)\) and so \(r\) is a skeptical consequence. This outcome corresponds to our approach for the theory $\AT_1 = (\{p\Rightarrow r, q\Rightarrow r\},\{p\vee q\})$: the argument \(\langle\langle p \vee q\rangle, [\langle p\rangle \Rightarrow r], [\langle q\rangle \Rightarrow r] \leadsto r\rangle\) is in all complete extensions of \(\AT_1\).

Next, consider Poole's \emph{broken arm} example \cite{Poole_KR_1989}. Let \(l\) be ``having a left broken arm'', \(r\) ``having a right broken arm'', \(w\) ``writing legibly''. On our approach, the theory $\AT_{\rm arm} = (\{w\Rightarrow\neg r\},\{l\vee r,w\})$ gives rise to the argument \(\langle\langle\langle w \rangle \Rightarrow \neg r\rangle, \langle l \vee r\rangle \rightarrow l\rangle\), which is in all complete extensions of $\AT_{\rm arm}$. In contrast, the disjunctive default theory \(\Delta_{\rm arm} = \lbrace l \mid r, w, w\Rightarrow \neg r \rbrace\) has two extensions \(Cn(\lbrace l, w, \neg r \rbrace)\) and \(Cn(\lbrace r, w \rbrace)\). Since \(l \notin Cn(\lbrace r,w \rbrace)\), \(l\) is not a skeptical consequence in disjunctive default logic.

Finally, reconsider the AT from Example \ref{motivation_att_1}. There are various ways in which we can translate this AT into a disjunctive default theory, e.g.\footnote{Our discussion also applies if we translate $p \Rightarrow q \vee r$ by $\frac{p ~:~ q \wedge r}{q \mid r}$.}
\[\Delta_3 = \left\lbrace \frac{p : q \vee r}{q \mid r}, \frac{q : s}{s}, \frac{s : v}{v}, \frac{r : v}{v}, \frac{t : \neg s}{\neg s},p , t \right\rbrace\]

% \begin{itemize}[itemsep=1pt,label=]
% \item \(\Delta_3' = \lbrace \frac{p : q \vee r}{q \mid r}, \frac{q : s}{s}, \frac{s : v}{v}, \frac{r : v}{v}, \frac{t : \neg s}{\neg s},p , t \rbrace\), and
% \item \(\Delta_3'' = \lbrace \frac{p : q \wedge r}{q \mid r}, \frac{q : s}{s}, \frac{s : v}{v}, \frac{r : v}{v}, \frac{t : \neg s}{\neg s}, p, t \rbrace\).
% \end{itemize}

For \(\Delta_3\) we have the extensions \(Cn(\lbrace p, t, q, s, v \rbrace)\), \(Cn(\lbrace p, t, q, \neg s \rbrace)\), and \(Cn(\lbrace p, t, \neg s, r, v \rbrace)\), so \(v\) is not a skeptical consequence. This corresponds to our approach in which the argument $A_1$ from Example \ref{motivation_att_1} is excluded due to the attack by $A_2$. However, on our approach $A_2$ is in all complete extensions, so contrary to disjunctive default logic we obtain $\neg s$ as a skeptical consequence.

%We continue our comparison with a similarity. For this consider the premise set \(\Sigma_3 = \lbrace p \Rightarrow q \vee r, q \Rightarrow s, s \Rightarrow v, r \Rightarrow v, t \Rightarrow \neg s, p, t \rbrace\). Again, there are various ways in which we can translate this into a disjunctive default theory, e.g. 
%
%\begin{itemize}
%\item \(\Delta_3' = \lbrace \frac{p : q \vee r}{q \mid r}, \frac{q : s}{s}, \frac{s : v}{v}, \frac{r : v}{v}, \frac{t : \neg s}{\neg s},p , t \rbrace\), and
%\item \(\Delta_3'' = \lbrace \frac{p : q \wedge r}{q \mid r}, \frac{q : s}{s}, \frac{s : v}{v}, \frac{r : v}{v}, \frac{t : \neg s}{\neg s}, p, t \rbrace\).
%\end{itemize}
%
%For \(\Delta_3'\) and \(\Delta_3''\) we have the extensions, \(Cn(\lbrace p, t, q, s, v \rbrace)\), \(Cn(\lbrace p, t, q, \neg s \rbrace)\) and \(Cn(\lbrace p, t, \neg s, r, v \rbrace)\). This means that \(v\) isn't a skeptical consequence. This again corresponds to our approach in which the argument 
%  \[ p \Rightarrow q \vee r, [\langle q \Rightarrow s\rangle \Rightarrow v], [r \Rightarrow v] \leadsto v \]
%cannot be defended against the rebut by \(t \Rightarrow \neg s\). 
%
%Finally, note that \(t \Rightarrow \neg s\) is grounded and thus \(\neg s\) is derivable in our approach, while it is not a skeptical consequence in disjunctive default logic.

\subsection{The OR meta-rule}
\label{sec:orgheadline2}

A different approach for dealing with disjunctive information is to allow for inference rules that produce new conditionals from given conditionals. We, for instance, find the following rule in system \textbf{P} \cite{KrLeMa90} and in several Input/Output logics \cite{MakTor00}:
\[ \frac{\psi \Rightarrow \phi ~~ \psi' \Rightarrow \phi}{\psi \vee \psi' \Rightarrow \phi}~~[{\rm OR}]\]
One could define an $ASPIC^+$-like system where arguments are constructed as in rules 1--3 in Definition \ref{def_arg}, and in which the defeasible rules are closed under OR. For instance, given \(\AT_1\) from Section \ref{sec:orgheadline1} this allows to derive \(p \vee q \Rightarrow r\) from \(p \Rightarrow r\) and \(q \Rightarrow r\), so that we can construct the argument \(\langle p \vee q\rangle \Rightarrow r\) and obtain \(r\) as a consequence.

Adding OR is not sufficient to always get the intuitive outcome. Let \(\AT_4 = (\{p \Rightarrow q \vee r, q \Rightarrow s, s \Rightarrow v, r \Rightarrow u, u \Rightarrow v\},\{p\})\).
One would want to build an argument for \(v\), but we cannot put OR to much use (except for deriving \(s \vee u \Rightarrow v\)). We would have to combine OR with e.g.\ right-weakening (RW: \(\frac{\psi \vdash \phi ~~ \psi' \Rightarrow \psi}{\psi' \Rightarrow \phi}\)), or generalize OR to
\[ \frac{\psi \Rightarrow \phi ~~ \psi' \Rightarrow \phi'}{\psi \vee \psi' \Rightarrow \phi \vee \phi'}~~[{\rm gOR}]\]
in order to produce \(q \vee r \Rightarrow s \vee u\). Since also \(s \vee u \Rightarrow v \vee v\) can be derived, we now have the means to construct the argument \(\langle \langle \langle p \Rightarrow q \vee r\rangle \Rightarrow s \vee u\rangle \Rightarrow v \vee v\rangle \rightarrow v\). In many systems of nonmonotonic logic, e.g.\ in system \textbf{P} and in many Input/Output logics, OR and RW are available (and thus gOR is a derived rule). We now contrast our approach with such (g)OR-based approaches.

A striking difference concerns the handling of Example \ref{motivation_att_1}. In the gOR-based system we can construct:
\[ A_3 = \langle \langle \langle p \Rightarrow q \vee r\rangle \Rightarrow s \vee v\rangle \Rightarrow v \vee v\rangle \rightarrow v \]
This argument is not attacked by the argument \(t \Rightarrow \neg s\). An alternative argument for \(v\) is given by \(\langle t \Rightarrow \neg s\rangle, \langle \langle p \Rightarrow q \vee r\rangle \Rightarrow s \vee v\rangle \rightarrow v\). Recall that neither in our approach nor in disjunctive default logic \(v\) is a skeptical consequence. Even if we add \(\neg r\) to the knowledge base $\mathcal{F}$ in Example \ref{motivation_att_1}, \(v\) remains derivable in the gOR-based approach since \(A_3\) remains unchallenged. This is counter-intuitive: both the argumentative path via \(q \Rightarrow s \Rightarrow v\) and the path via \(r \Rightarrow v\) are barred in view of the unchallenged arguments \(t \Rightarrow \neg s\) and \(\neg r\). An advantage of using the reasoning-by-cases rule 4 in Definition \ref{def_arg} is that it provides more fine-grained ways of tracking commitments in sub-arguments. This enables us to block the undesired consequence \(v\) in this example.\footnote{An additional advantage is that argument strength can be tracked in a more fine-grained way when using RbC in contrast to OR-based approaches. See also Section \ref{sec:conclusion-outlook}. % We will investigate argument strength in more detail in a follow-up paper, in which we discuss rbc-arguments in the presence of priorities on arguments
}

%%% Local Variables:
%%% mode: latex
%%% TeX-master: "jelia-rbc"
%%% End:

\section{Conclusion and outlook}
\label{sec:conclusion-outlook}
The ideas developed in this paper offer many interesting avenues for further work, partially consisting of the removal of the limitations assumed in this paper. For example, we did not consider nested rbc-arguments or various components that can be modelled in the ASPIC$^+$-framework such as defeasible premises, undermining or undercut attacks. We also plan to present a less cautious variation of this framework where an attack on an rbc-argument $A, [B_1],\ldots,[B_n]\leadsto \psi$ succeeds only if each of the $B_i$'s is attacked. In addition we will investigate prioritized default rules in this framework. Here too, new questions arise. Consider, for instance, an argument $\langle p\vee q \rangle, [\langle\langle p\rangle \Ra s\rangle \Ra t], [\langle q\rangle \Ra t] \leadsto t$. If in addition there is an argument $\langle\top\rangle\Ra\neg s$ which is preferred over $\langle p\rangle\Ra s$, then it seems intuitive to let the former attack the latter argument. But what if the hypothetical argument $\langle p\rangle\Ra s$ has a higher degree of priority than the non-hypothetical $\langle\top\rangle\Ra\neg s$? Should we decide in favor of the highest priority assigned, or should we never let a hypothetical argument attack a non-hypothetical one? More generally, how do we lift the priorities assigned to the (hypothetical and non-hypothetical) constituents of an rbc-argument? Difficult questions like these will have to be answered in order to resolve conflicts between prioritized rbc-arguments.

In future work we also plan to investigate the use of an ordered disjunction $\overleftarrow{\lor}$ (see, e.g., \cite{Brewka02}). For instance, a rule $p \Rightarrow q \overleftarrow{\lor} r$ can be read as: `If $p$ then plausibly $q$ or $r$, where $q$ is more plausible than $r$'. This is especially interesting when measures of argument strength are considered and when thinking about defeat of rbc-arguments based on ordered disjunctions such as $\langle\langle q \overleftarrow{\lor} r\rangle, [q \Rightarrow s], [r \Rightarrow s] \leadsto s\rangle$.

\bibliographystyle{plain}
\bibliography{phdrefs}

\section*{APPENDIX}
\label{sec:orgheadline3}

For the meta-proofs below it will be  sometimes useful to speak about sub-arguments in the following stronger sense. \(\Sub'(\langle A_{1}, [A_2], \ldots, [A_n] \leadsto \phi \rangle)\) consists of \(\Sub'(A_1)\) and  \(\langle A_{1}, [A_2'], \ldots, [A_n'] \leadsto \bigvee \{ \Con(A_i') \mid 2 \le i \le n\} \rangle\) for every \(A_i' \in \Sub(A_i)\) where \(2 \le i \le n\). %
For other types of arguments \(A\) (see Def.~\ref{def_AT}, items 1-3), the definition of \(\Sub'(A)\) is the same as the definition of \(\Sub(A)\).
Note that for any argument \(A\), \(\Sub(A) \subseteq \Sub'(A)\).
In the following results, if not specified further, $\SAF$ is a structured argumentation framework based on the arbitrary argumentation theory $\AT$.
The next theorem is a strong form of sub-argument closure:
\begin{theorem}
  \label{thm:closure:sub:prime}
If \(\Ec \in {\sf Cmp}(\SAF)\) and \(A \in \Ec\), \(\Sub'(A) \subseteq \Ec\).
\end{theorem}
\begin{proof}
Let \(B \in \Sub'(A)\). It is easy to see that every attacker \(C\) of \(B\) attacks \(A\). Since \(\Ec\) is complete, it defends \(A\) and hence also \(B\). Again, since \(\Ec\) is complete, \(B \in \Ec\).
\end{proof}

%\begin{corollary}[Sub-argument closure]
%\label{cor:subarg-clo}
%If \(\Ec \in {\sf Cmp}(\AT)\) and \(A \in \Ec\), then \(\Sub(A) \subseteq \Ec\).
%\end{corollary}

\begin{definition}
Where \(\Sub'(A) = \lbrace A_1, \ldots, A_n \rbrace\), let \(\hat{A} = \langle A_1, \ldots, A_n \rightarrow \bigwedge_{i=1}^n \Con(A_i) \rangle\).
\end{definition}

\begin{lemma}
\label{fact:A:hatA:attackers}\label{fact:A:hatA:dagger} \label{lem:A:hatA:complete} \label{lem:A:hatA:complete:2}
(1) \(A\) and \(\hat{A}\) have the same attackers in \(\SAF\); (2) \(\vdash \dagger A \equiv \dagger \hat{A}\); (3) Where \(\Ec \in {\sf Cmp}(\SAF)\), (3.1) if \(A \in \Ec\), \(\hat{A} \in \Ec\); (3.2) if \(\langle A_1, \ldots, A_n \rightarrow \phi \rangle \in \Ec\) and $\Con(\hat{A_1}), \ldots, \Con(\hat{A_n})\vdash \psi$, \(\langle \hat{A_{1}}, \ldots, \hat{A_n} \rightarrow \psi \rangle \in \Ec\).
\end{lemma}
\begin{proof}
  (1) follows due to the fact that the set of defeasible conclusions of \(A\) is the same as the set of defeasible conclusions of \(\hat{A}\). (2) follows by simple $\mathsf{CL}$-manipulations in view of Definition~\ref{def_dagger}. For (3.1) note that, by (1), if some \(B\) attacks \(\hat{A}\) then it also attacks \(A\). Since \(\Ec\) is complete, \(\hat{A}\) is defended and thus in \(\Ec\). The proof of (3.2) is similar to (3.1) and left to the reader.
\end{proof}

% \begin{lemma}

% Where \(\Ec \in {\sf Cmp}(\AT)\), \(\langle A_1, \ldots, A_n \rightarrow \phi \rangle \in \Ec\): also \(\langle \hat{A_{1}}, \ldots, \hat{A_n} \rightarrow \psi \rangle \in \Ec\) if $\Con(\hat{A_1}), \ldots, \Con(\hat{A_n})\vdash \psi$.
% \end{lemma}
% The proof is similar to the proof of Lemma \ref{lem:A:hatA:complete}.3.

\begin{theorem}[Closure under strict rules]
\label{thm:closure}
Where \(A_1, \ldots, A_n \in \Ec \cap \Arg(\AT)\) and \(\Ec \in {\sf Cmp}(\SAF)\):  \(\langle A_1, \ldots, A_n \rightarrow \phi \rangle \in \Ec \cap \Arg(\AT)\) if \(\Con(A_1), \ldots, \Con(A_n) \vdash \phi\).
\end{theorem}
\begin{proof}
Let $\AT = (\Dc,\Fc)$. We show the statement by an induction over \(s = \sum_{i=1}^n \# A_i\) where \(\# B\) is the number of rules from $\Dc$ and ${\cal S}$ applied in \(B\). We have to show two things:
(\textbf{i}) \(A \in \Arg(\AT)\) and (\textbf{ii}) \(A \in \Ec\).

(\(s = 0\)) Then \(A\) has the form \(\langle \langle \phi_1 \rangle , \ldots, \langle \phi_n \rangle \rightarrow \phi\rangle\). Thus, \(\phi_1, \ldots, \phi_n \in \Fc\) and \(\phi_1, \ldots, \phi_n \vdash \phi\).  Since we assume our knowledge base to be consistent, \(\lbrace \phi_1, \ldots, \phi_n, \phi \rbrace\) is consistent and hence \(A \in \Arg(\AT)\). Since \(A\) has no attackers and \(\Ec\) is complete, \(A \in \Ec\). 

(\(s {\Rightarrow} s{+}1\)) % Let \(A = \langle A_{1}, \ldots, A_{n} \rightarrow \phi\rangle\) where \(\Con(A_{1}), \ldots, \Con(A_n) \vdash \phi\).
We first show (\textbf{i}). Assume for a contradiction that \(A \in \Arg^{\bot}(\AT) \setminus \Arg(\AT)\). This means that \(\dagger A_{1}, \ldots, \dagger A_{n}, \phi \vdash \bot\). Since \(\Con(A_{1}), \ldots, \Con(A_{n}) \vdash \phi\) this implies by transitivity that \(\dagger A_{1}, \ldots, \dagger A_{n} \vdash \bot\). Let \(n'\) be minimal such that \(1 \le n' \le n\) and \(\dagger A_{1}, \ldots, \dagger A_{n'} \vdash \bot\). By the induction hypothesis, \(\sum_{i=1}^{n'{-}1} \# A_i \le s < \sum_{i=1}^{n'} \# A_i\) which implies \(\# A_{n'} \ge 1\). Let \(\langle B_{1}, \ldots, B_m \rangle\) be a list of $\Sub'(A_{n'})$ ordered in such a way that for all \(1 \le i < j \le m\) we have \(B_j \notin \Sub'(B_i)\) and \(\Sub'(B_j) \subseteq \lbrace B_1, \ldots, B_j \rbrace\). Take a minimal \(k\) such that \(1 \le k \le m\) and \(\dagger A_{1}, \ldots, \dagger A_{n'-1}, \dagger B_{1}, \ldots, \dagger B_k \vdash \bot\). Then, 
\begin{align}
\label{eq:closure:1}
  \dagger A_{1}, \ldots, \dagger A_{n'-1}, \dagger B_{1}, \ldots, \dagger B_{k-1} \vdash \neg \dagger B_k \\
  \label{eq:closure:2}
  \dagger A_{1}, \ldots, \dagger A_{n'-1}, \dagger B_{1}, \ldots, \dagger B_{k-1} \nvdash \bot \\
\label{eq:closure:4}
\dagger A_{1}, \ldots, \dagger A_{n'-1}, \dagger B_{1}, \ldots, \dagger B_{k-1} \vdash \neg \Con(B_k)  
\end{align}
\eqref{eq:closure:4} follows by simple $\mathsf{CL}$-manipulations in view of \eqref{eq:closure:1} and since \(\dagger B_k = \dagger B_{j_1} \wedge \ldots \wedge \dagger B_{j_o} \wedge \Con(B_k)\). We distinguish the following cases concerning the form of \(B_{k}\):

\begin{enumerate}[fullwidth,itemsep=0pt,topsep=0pt]
\item \(B_{k} = \langle B_{j_1}, \ldots, B_{j_o} \twoheadrightarrow \phi_k \rangle\) where \(\{ B_{j_1}, \ldots, B_{j_o} \} \subseteq \lbrace B_1, \ldots, B_{k-1} \rbrace\) and ${\twoheadrightarrow} \in \{{\rightarrow}, {\Rightarrow}\}$, or
% \item \(B_{k} = \langle B_{j_1}, \ldots, B_{j_o} \Rightarrow \phi_k \rangle\) where \(\{ B_{j_1}, \ldots, B_{j_o} \} \subseteq \lbrace B_1, \ldots, B_{k-1} \rbrace\);
\item \(B_{k} = \langle B_{j_1}, [B_{j_2}], \ldots, [B_{j_o}] \leadsto \phi_k \rangle\) where for each \(B_{j_2}' \in \Sub'(B_{j_2})\), \ldots{}, each \(B_{j_o}' \in \Sub'(B_{j_o})\) s.t.\ \(\lbrace B_{j_2}', \ldots, B_{j_o}' \rbrace \neq \lbrace B_{j_2}, \ldots, B_{j_o} \rbrace\),  \(\langle B_{j_{1}}, [B_{j_2}'], \ldots, [B_{j_o}'] \leadsto \bigvee \{ \Con(B_{j_i}' \mid j_2 \le i  \le j_o \}) \rangle \in \lbrace B_1, \ldots, B_{k-1} \rbrace\)
\end{enumerate}

Consider case 1 with ${\twoheadrightarrow} = {\Rightarrow}$ (the case ${\twoheadrightarrow} = {\rightarrow}$ is left to the reader). % Since \(\Con(B_{j_1}), \ldots, \Con(B_{j_o}) \vdash \phi_k\), and for each \(1 \le g \le j_o\), \(\dagger B_{j_{g}} \vdash \Con(B_{j_g})\), also 
% \begin{equation}
% \label{eq:closure:3}
% \dagger A_{1}, \ldots, \dagger A_{n'-1}, \dagger B_{1}, \ldots, \dagger B_{k-1} \vdash \phi_k. 
% \end{equation}
% In view of \eqref{eq:closure:3} and \eqref{eq:closure:4}, \(\dagger A_{1}, \ldots, \dagger A_{n'-1}, \dagger B_{1}, \ldots, \dagger B_{k-1} \vdash \bot\) which contradicts \eqref{eq:closure:2}. \medskip
% \emph{Ad 2:}
By the inductive hypothesis, \eqref{eq:closure:4}, Theorem \ref{thm:closure:sub:prime} and Lemma \ref{lem:A:hatA:complete}, % and \ref{lem:A:hatA:complete:2},
$A' = \langle \hat{A_{1}}, \ldots, \hat{A_{n-1}}, \hat{B_1}, \ldots, \hat{B_{k-1}} \rightarrow \neg \Con(B_k) \rangle \in \Ec$.
However, since also \(B_k \in \Ec\) and \(A'\) attacks \(B_{k}\), this is a contradiction to the conflict-freeness of \(\Ec\). 

We now sketch the proof for case 2. By some basic $\mathsf{CL}$-manipulations it can be shown that for all \(j_2 \le i \le j_{o}\)
% Note that for each \(B_{j_2}' \in \Sub'(B_{j_2})\), \ldots{}, and each \(B_{j_o}' \in \Sub'(B_{j_o})\) such that \(\lbrace B_{j_2}', \ldots, B_{j_o}' \rbrace \neq \lbrace B_{j_2}, \ldots, B_{j_o} \rbrace\),
% \begin{multline}
% \label{eq:closure:31}
% \dagger A_{1}, \ldots, \dagger A_{n'-1}, \dagger B_{1}, \ldots, \dagger B_{k-1} \vdash \\
% \dagger B_{j_1} \wedge (\dagger B_{j_2}' \vee \ldots \vee \dagger B_{j_o}') \wedge \bigvee \{ \Con(B_{j_i}' \mid j_2 \le i \le j_o \})
% \end{multline}
% By \eqref{eq:closure:4}, 
% \begin{equation}
% \dagger A_{1}, \ldots, \dagger A_{n'-1}, \dagger B_{1}, \ldots, \dagger B_{k-1} \vdash \neg\dagger B_{j_1} \vee \neg\phi_k \vee (\neg\dagger B_{j_2} \wedge \ldots \wedge \neg\dagger B_{j_o}) 
% \end{equation}
% By \eqref{eq:closure:31},
% \begin{equation}
% \label{eq:closure:32}
% \dagger A_{1}, \ldots, \dagger A_{n'-1}, \dagger B_{1}, \ldots, \dagger B_{k-1} \vdash \neg\phi_k \vee (\neg\dagger B_{j_2} \wedge \ldots \wedge \neg\dagger B_{j_o}) 
% \end{equation}
% Since by \eqref{eq:closure:31}, for every \(k\) for which \(j_2 \le i \le j_o\) and for every \(B_{j_i}' \in \Sub'(B_{j_i})\), 
% \begin{multline}
% \dagger A_{1}, \ldots, \dagger A_{n'-1}, \dagger B_{1}, \ldots, \dagger B_{k-1} \vdash \\
% \dagger B_{j_2} \vee \ldots \vee \dagger B_{j_i-1} \vee \dagger B_{j_i}' \vee \dagger B_{j_i+1} \vee \ldots \vee \dagger B_{j_o}
% \end{multline}
% we get
% \begin{equation}
% \dagger A_{1}, \ldots, \dagger A_{n'-1}, \dagger B_{1}, \ldots, \dagger B_{k-1} \vdash \neg\phi_k \vee \neg \Con(B_{j_i})
% \end{equation}
% Since \(\neg\Con(B_{j_i}) \vdash \neg\phi_k\),
\begin{equation}
\label{eq:closure:33}
\dagger A_{1}, \ldots, \dagger A_{n'-1}, \dagger B_{1}, \ldots, \dagger B_{k-1} \vdash \neg\Con(B_{j_i})
\end{equation}
We distinguish two cases:
(a) for all $j_2 \le i \le l_o$, \(B_{j_i}\) is of the form \(\langle \ldots \rightarrow \Con(B_{j_i})\rangle\) and (b) there is a \(j_2 \le l \le j_o\) for which \(B_{j_l}\) is of the form \(\langle \ldots \Rightarrow \Con(B_{j_l}) \rangle\).
In case (a), $\dagger A_{1}, \ldots, \dagger A_{n'-1}, \dagger B_{1}, \ldots, \dagger B_{k-1} \vdash \bigvee \{\Con(B_{j_i}) \mid j_2 \le i \le j_o \}$, which with \eqref{eq:closure:33} contradicts \eqref{eq:closure:2}. %
In case (b), the argument 
$C = \langle \hat{A_{1}}, \ldots, \hat{A_{n'-1}}, \hat{B_1}, \ldots, \hat{B_{k-1}} \rightarrow \neg \Con(B_{j_l}) \rangle$
attacks \(B_k\) in its hypothetical
subargument \(B_{j_l}\). By the inductive hypothesis, \eqref{eq:closure:4}, Theorem \ref{thm:closure:sub:prime}, and Lemma \ref{lem:A:hatA:complete}, % and \ref{lem:A:hatA:complete:2},
\(C \in \Ec\). This contradicts the conflict-freeness of \(\Ec\) and completes the proof of claim (\textbf{i}).

Suppose that some \(C \in \Arg(\AT) \cup \HArg(\AT)\) attacks \(A\). Clearly, \(C\) attacks some \(A_i\) (where \(1 \le i \le n\)). Since, by Theorem \ref{thm:closure:sub:prime}, \(A_i \in \Ec\), some \(D \in \Ec\) attacks \(C\). Altogether this shows that \(\Ec\) defends \(A\) and since \(\Ec\) is complete, \(A \in \Ec\). This is claim (\textbf{ii}).
\end{proof}

\begin{theorem}[Consistency]
\label{thm:indirect:con}
Where \(\Ec \in {\sf Cmp}(\AT)\) and \(A_1, \ldots, A_n \in \Ec \cap \Arg(\AT)\), \(\lbrace \Con(A_1), \ldots, \Con(A_n) \rbrace\) is consistent.
\end{theorem}
\begin{proof}
% Let \(A_1, \ldots, A_n \in \Ec\) where \(\Ec\) is a complete extension.
By Theorem \ref{thm:closure}, \(A = \langle A_{1}, \ldots, A_n \rightarrow \top \rangle \in \Ec \cap \Arg(\AT)\) which implies that \(\lbrace \Con(A_1), \ldots, \Con(A_n) \rbrace\) is consistent since otherwise \(\dagger A \vdash \bot\).
\end{proof}

% \begin{theorem}[Non-Interference]\footnote{Due to space restrictions we omit the proof.}
%   Where ${\nc} \in \lbrace {\nc_{\sf grd}}, {\nc_{\sf prf}^{\cap}}, {\nc_{\sf prf}^{\Cap}} \rbrace$, 
%   $({\sf Atoms}(\Dc) \cup {\sf Atoms}(\Fc)) \cap ({\sf Atoms}(\Dc') \cup {\sf Atoms}(\Fc')) = \emptyset$, and ${\sf Atoms}(\phi) \subseteq {\sf Atoms}(\Dc) \cup {\sf Atoms}(\Fc)$, we have: $(\Dc,\Fc) \nc \phi$ iff $(\Dc \cup \Dc', \Fc \cup \Fc') \nc \phi$.
% % \begin{multline*}
% %   Cn_{\nc}(\Dc \cup \Dc', \Fc \cup \Fc')_{\mid {\sf Atoms}(\Dc) \cup {\sf Atoms}(\Fc)} = \\
% %   Cn_{\nc}(\Dc, \Fc)_{\mid {\sf Atoms}(\Dc) \cup {\sf Atoms}(\Fc)}.
% % \end{multline*}
% \end{theorem}

% \begin{corollary}[Direct Consistency]
% \label{cor:direct:con}
% Where \(\Ec\) is a complete extension, \(\lbrace \Con(A) \mid A \in \Ec \rbrace\) is consistent.
% \end{corollary}

%%% Local Variables:
%%% mode: latex
%%% TeX-master: "jelia-rbc"
%%% End:

\end{document}